\documentclass[twoside]{article}

\usepackage[accepted]{aistats2025}
\usepackage{booktabs}
\usepackage{natbib}
\usepackage{amsmath,amssymb}
\usepackage{mathtools}

\usepackage{amsmath,amsfonts,bm}

\def\eqref#1{equation~\ref{#1}}

\def\1{\bm{1}}

\DeclareMathAlphabet{\mathsfit}{\encodingdefault}{\sfdefault}{m}{sl}
\SetMathAlphabet{\mathsfit}{bold}{\encodingdefault}{\sfdefault}{bx}{n}

\usepackage{hyperref}
\usepackage{url}
\usepackage{graphics}
\usepackage{graphicx}
\usepackage{subfigure}
\usepackage{caption}
\usepackage{wrapfig}
\usepackage{amsthm}
\usepackage{multirow}
\usepackage{multicol}
\usepackage[ruled,vlined]{algorithm2e}
\usepackage{pythontex}

\newtheorem{theorem}{Theorem}
\newtheorem{lemma}[theorem]{Lemma}

\newcommand{\xx}{\boldsymbol{x}}
\newcommand{\zz}{\boldsymbol{z}}

\newcommand{\diff}{\mathrm{d}}
\newtheorem*{theorem*}{Theorem}
\newtheorem*{proposition*}{Proposition}

\usepackage[dvipsnames]{xcolor}
\definecolor{darkblue}{rgb}{0.0, 0.0, 0.55}
\definecolor{dark2blue}{rgb}{0.0, 0.0, 0.4}
\definecolor{darkred}{rgb}{0.55, 0.0, 0.0}
\definecolor{darkgreen}{rgb}{0, 0.5, 0.0}
\usepackage{hyperref}
\hypersetup{colorlinks=true,citecolor=darkblue,linkcolor=red}

\begin{document}

%

%

\twocolumn[

\aistatstitle{Denoising Fisher Training For Neural Implicit Samplers}

\aistatsauthor{ Weijian Luo \And Wei Deng}

\aistatsaddress{ Peking University\\ \texttt{luoweijian@stu.pku.edu.cn}\\ \texttt{pkulwj1994@icloud.com} \And  Machine Learning Research,\\ Morgan Stanley, New York\\ \texttt{weideng056@gmail.com}}]

\begin{abstract}
Efficient sampling from un-normalized target distributions is pivotal in scientific computing and machine learning. While neural samplers have demonstrated potential with a special emphasis on sampling efficiency, existing neural implicit samplers still have issues such as poor mode covering behavior, unstable training dynamics, and sub-optimal performances. To tackle these issues, in this paper, we introduce Denoising Fisher Training (DFT), a novel training approach for neural implicit samplers with theoretical guarantees. We frame the training problem as an objective of minimizing the Fisher divergence by deriving a tractable yet equivalent loss function, which marks a unique theoretical contribution to assessing the intractable Fisher divergences. DFT is empirically validated across diverse sampling benchmarks, including two-dimensional synthetic distribution, Bayesian logistic regression, and high-dimensional energy-based models (EBMs). Notably, in experiments with high-dimensional EBMs, our best one-step DFT neural sampler achieves results on par with MCMC methods with up to 200 sampling steps, leading to a substantially greater efficiency over 100 times higher. This result not only demonstrates the superior performance of DFT in handling complex high-dimensional sampling but also sheds light on efficient sampling methodologies across broader applications.
\end{abstract}

\section{INTRODUCTION}
Efficiently drawing samples from un-normalized distributions is a fundamental issue in various domains of research, including Bayesian inference \citep{revjump}, simulations in biology and physics \citep{mcprotein,mcphysics}, and the fields of generative modeling and machine learning \citep{Xie2016ATO,andrieu2003introduction}. The goal is to generate sets of samples from a target distribution defined by a differentiable, but possibly not normalized, potential function, denoted as $\log q(\xx)$. The main challenge lies in ensuring the accuracy of the target distribution with minimal sampling costs, especially in simulations involving multi-modal distributions. 

\begin{figure*}
\centering
\includegraphics[width=1.0\linewidth]{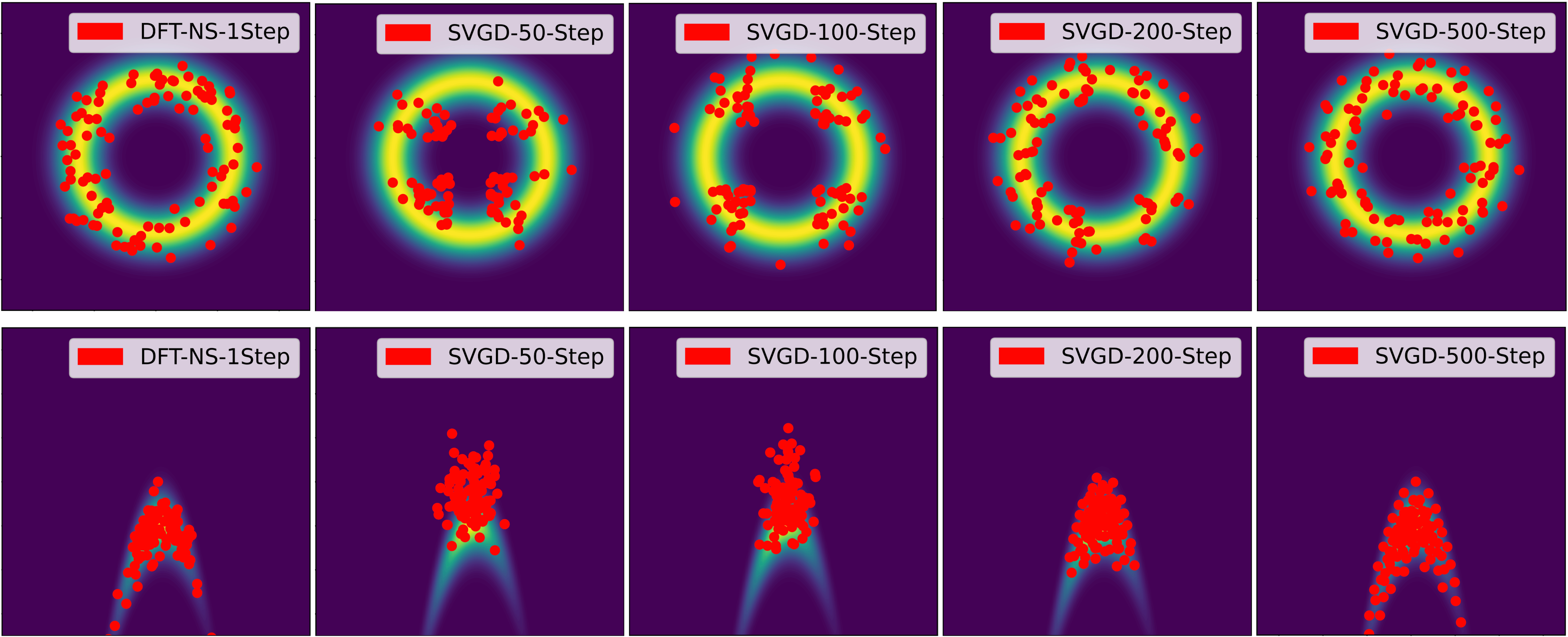}
\caption{Visualizations of DFT neural sampler against Stein Variational Gradient Decent \citep{liu2016stein} with multiple sampling steps. The DFT-NS with 1 sampling step outperforms SVGD with 200 sampling steps. \textbf{Upper}: Donut target distribution in Table \ref{tab:1}; \textbf{Under}: Rosenbrock target distribution in Table \ref{tab:1}.}
\label{fig:big_vis}
\end{figure*}

The sampling challenge can be tackled using two main methodologies. The first involves a range of Markov Chain Monte Carlo (MCMC) techniques \citep{mh, ld, mala, hmc}. These techniques create Markov Chains that are designed to have the desired target distribution as their stationary distribution. By simulating these chains, random noises can be iteratively adjusted to stationary distributions following the target distribution. While MCMC methods are renowned for generating asymptotically unbiased samples, they often face computational inefficiencies, particularly when dealing with barriers in high-dimensional multi-modal distributions \citep{huang2024reverse}. This is largely due to the large number of iterations necessary to produce a batch of samples that accurately represent the target distribution.

The second category of methods is known as learning to sample (L2S) models \citep{Lvy2017GeneralizingHM,nss,snf,Langosco2021NeuralVG,arbel2021annealed,Zhang2021PathIS,matthews2022continual,vargas2022denoising,lahlou2023theory}. These models leverage modern neural networks for sampling, aiming to improve both the quality and efficiency of sample generation compared to traditional non-learning-based methods such as MCMCs.

A prominent example within the L2S framework is the neural implicit sampler, which is recognized for its exceptional efficiency. An implicit neural sampler employs a transformation parametrized with a neural network to directly transform an easy-to-sample initial distribution to obtain final samples without multiple iterations. This can bypass the crucial challenge of sampling from multi-modal distributions via MCMCs and propose a direct transport from simple distributions to the target \citep{huang2024reverse}. 
As a result, the neural implicit sampler offers significant enhancements in efficiency and scalability using just one neural network pass. They are particularly beneficial in high-dimensional spaces, such as those encountered in image generations.

In this paper, we introduce Denoising Fisher Training (DFT), a novel approach for training neural implicit samplers to efficiently sample from un-normalized target distributions. We frame the training objective by minimizing the Fisher divergence between the implicit sampler and the target distribution. Though the direct objective is intractable, we introduce a tractable yet equivalent training objective in Section \ref{sec:objective} with theoretical guarantees. Our theoretical assessments also introduce new tools for handling intractable Fisher divergences in broader applications. To evaluate the performance, efficiency, and scalability of DFT and the corresponding neural samplers, we have carried out empirical evaluations across three different sampling benchmarks with varying degrees of complexity in Section \ref{sec:exp}. The benchmarks include:
\begin{itemize}
    \item Sampling from two-dimensional toy distributions;
    \item Bayesian inference with moderate dimensionality;
    \item Sampling from high-dimensional energy-based models (EBMs) using the MNIST dataset.
\end{itemize}
The experimental results consistently show that DFT samplers outperform existing methods in terms of sample quality. Notably, in the high-dimensional EBM tests, our neural samplers produced sample quality on par with the baseline EBM but with computational efficiency more than 200 times greater than traditional MCMC methods. These findings highlight the effectiveness, efficiency, and versatility of our proposed methods across a wide range of sampling scenarios.

\section{RELATED WORKS}
\subsection{Neural Samplers}\label{sec:gm_stein}
Neural networks have emerged as powerful tools in various domains, demonstrating their ability to produce diverse, high-quality samples. They have been successfully applied in tasks such as text-to-image generation \citep{brock2018large,karras2019style,karras2020analyzing,karras2021alias,nichol2021improved,dhariwal2021diffusion,ramesh2022hierarchical,saharia2022photorealistic,rombach2022high}, audio generation\citep{huang2023make}, video and 3D creation\citep{clark2019adversarial,ho2022imagen,molad2023dreamix,poole2022dreamfusion}, and even molecule design \citep{Nichol2021GLIDETP,Ho2022VideoDM}. Recently, there has been a growing interest in leveraging neural networks for sampling from target distributions. 

Three primary classes of neural networks have been extensively studied for sampling tasks. The first class comprises normalizing flows (NFs) \citep{rezende2015variational}, while the second class consists of diffusion models (DMs) \citep{song2020score}. NFs employ invertible neural transformations to map Gaussian latent vectors $\zz$ to obtain samples $\xx$. The strict invertibility of NF transformations enables the availability of likelihood values for generated samples, which are differentiable with respect to the model's parameters. Training NFs often involves minimizing the KL divergence between the NF and the target distribution \citep{snf}. On the other hand, DMs employ neural score networks to model the marginal score functions of a data-initialized diffusion process. 
DMs have been successfully employed to enhance the sampler performance of the annealed importance sampling algorithm \citep{ais01}, a widely recognized MCMC method for various sampling benchmarks. Despite the successes of NFs and DMs, both models have their limitations. The invertibility of NFs restricts their expressiveness, which can hinder their ability to effectively model high-dimensional targets. Moreover, DMs still require a considerable number of iterations for sample generation, resulting in computational inefficiency.

\subsection{Implicit Neural Samplers}
Different from other neural samplers, neural implicit samplers are favored for their high efficiency and flexible modeling ability. Typically, an implicit generative model uses a flexible neural transform $g_\theta(.)$ to push forward easy-to-sample noises $\zz\sim p_z$ to obtain samples $\xx=g_\theta(\zz)$. Though implicit samplers have multiple advantages, training them is not easy. 

Research has explored various techniques for efficiently training neural implicit samplers by focusing on minimizing divergence. For example, the study by \citet{nss} examines the training process through the reduction of the Stein discrepancy, while another study looks into practical algorithms designed to decrease the Kullback-Leibler divergence and the Fisher divergence. Although these approaches have shown promising results, they are not without their shortcomings. In particular, the methods intended to minimize Fisher divergence, despite their theoretical validity, have struggled to perform effectively with target distributions that exhibit high multi-modality\citep{luo2024entropy}. Besides, the training approach that minimizes the KL divergence also encounters sub-optimal performances \citep{luo2024entropy}. On the contrary, our introduced Denoising Fisher Training overcame the above issues by introducing a new stable and equivalent loss. Before we introduce the DFT, we give some background knowledge in Section \ref{sec:background}. 

\section{BACKGROUND}\label{sec:background}
\paragraph{Standard Score Matching.}
Score matching \citep{hyvarinen2005estimation} provided practical approaches to estimating score functions. Assume one only has available samples $\xx\sim p$ and wants to use a parametric approximated distribution $q_\phi(\xx)$ to approximate $p$. Such an approximation can be made by minimizing the Fisher Divergence between $p$ and $q_\phi$ with the definition
\begin{align*}
    \mathcal{D}_{FD}(p,q_\phi) \coloneqq &\mathbb{E}_{\xx\sim p} \biggl\{ \|\nabla_{\xx} \log p(\xx)\|_2^2 + \|\nabla_{\xx} \log q_\phi(\xx)\|_2^2 \\
    & - 2\langle\nabla_{\xx} \log p(\xx), \nabla_{\xx} \log q_\phi(\xx)\rangle \biggl\}.
\end{align*}
Under certain mild conditions, the objective equals to 
\begin{align}\label{eqn:sm_objective}
    \mathcal{L}(\phi) = \mathbb{E}_{p} \biggl\{& \|\nabla_{\xx} \log q_\phi(\xx)\|_2^2 + 2 \Delta_{\xx} \log q_\phi(\xx)  \biggl\}.
\end{align}
This objective can be estimated only through samples from $p$,  thus is tractable when $q_\phi$ is well-defined. Moreover, one only needs to define a score network $\bm{s}_\phi(\xx)\colon \mathbb{R}^D \to \mathbb{R}^D$ instead of a density model to represent the parametric score function. This technique was named after Score Matching. Other variants of score matching were also studied \citep{ssm, fdsm, meng2020autoregressive, lu2022maximum,bao2020bi}. 

\paragraph{Denoising Score Matching.}
Notice that the standard score-matching objective has a term $\Delta_{\xx} \log q_\phi(\xx)$ that is computationally expensive. Denoising score matching \citep{vincent2011connection} is introduced to remedy this issue. The concept of denoising score matching is to perturb the sampling distribution by adding small Gaussian noises and then letting the score network match the conditional score functions. Specifically, let $\xx\sim p$ be drawn from the sampling distribution. And let $\xx_\sigma = \xx + \sigma\epsilon, ~~ \epsilon\sim \mathcal{N}(\bm{0}, \mathcal{I})$. When $\sigma$ is small, the distribution $p_\sigma$ of $\xx_\sigma$ is a good approximation of the sampling distribution $p$ of $\xx_0$. We use $p(\xx_\sigma|\xx_0)$ to denote the conditional distribution. If we want to use a parametric distribution $q_\phi$ (or a neural score function $\bm{s}_{\phi}(.)$) to approximate the distribution of  The matching objective of denoising score matching is to minimize 
\begin{align}\label{eqn:dsm_objective}
    \nonumber
    \mathcal{L}_{DSM}(\phi) = & \mathbb{E}_{\xx_0\sim p, \atop \xx_\sigma|\xx_0\sim p(\xx_\sigma|\xx_0)} \bigg\{\|\nabla_{\xx_\sigma} \log q(\xx_\sigma)\\
    & - \nabla_{\xx_\sigma} \log p(\xx_\sigma|\xx_0) \|_2^2 \bigg\}.
\end{align}

\section{DENOISING FISHER TRAINING}
In this section, we introduce the denoising Fisher training (DFT) method and show its equivalence to minimizing the Fisher divergence. The DFT incorporates a noise-injection and denoising mechanism into the training of implicit sampler and shows significant stability than the previous Fisher training method in both theoretical and empirical aspects. 

\subsection{The Training Objective}\label{sec:objective}

\paragraph{The Problem Setup.}
Let $g_\theta(\cdot)\colon \mathbb{R}^{D_Z}\to \mathbb{R}^{D_X}$ be an implicit sampler (i.e., a neural network transform), $p_z$ the latent distribution, $p_\theta$ the sampler induced distribution $\xx=g_\theta(\zz)$, and $q(\xx)$ the un-normalized target. Our goal in this section is to train the implicit sampler $g_\theta$ such that $p_\theta(\xx)$ equals $q(\xx)$. Instead of directly minimizing the Fisher divergence between $p_\theta(\xx)$ and $q(\xx)$, we add some little noise to the implicit sampler to tweak the sampler distribution with 
\begin{align}
    \xx_\sigma \coloneqq \xx_0 + \sigma\epsilon, ~~\xx_0 = g_\theta(\zz), ~~ \epsilon\sim \mathcal{N}(\bm{0}, \mathcal{I})
\end{align}
Therefore the conditional distribution $p(\xx_\sigma|\xx_0) = \mathcal{N}(\xx_0, \sigma^2\mathcal{I})$ has an explicit form. When $\sigma$ is set to be small, the distribution $p_{\theta, \sigma}$ which represents the distribution of $\xx_\sigma$ is a sufficiently good approximation of the implicit distribution $p_\theta$ of $\xx_0$. We denote $\bm{s}_q(\xx_\sigma)\coloneqq \nabla_{\xx_\sigma} \log q(\xx_\sigma)$ and $\bm{s}_{\theta,\sigma}(\xx_\sigma)\coloneqq \nabla_{\xx_\sigma} \log p_{\theta,\sigma}(\xx_\sigma)$. Our goal is to minimize the Fisher divergence between $p_{\sigma,\theta}(\xx)$ and $q(\xx)$ which writes
\begin{align}\label{eqn:tFD}
    &\mathcal{D}_{FD}(\theta) \coloneqq \mathbb{E}_{\xx_\sigma \sim p_{\theta, \sigma}} \|\bm{s}_{q}(\xx_\sigma) - \bm{s}_{\theta, \sigma}(\xx_\sigma)\|_2^2 
\end{align}
\paragraph{The Intractabe Objective}
To make the derivation neat, we may use the notion $\xx_\sigma(\theta)$ and $\xx_\sigma$ interchangeably to emphasize the parameter dependence of $\xx_\sigma$ and $\theta$ if necessary. 
In order to minimize the objective \eqref{eqn:tFD}, we take the $\theta$ gradient for such an objective, which writes
\begin{align}\label{eqn:tFD_grad_full}
    \nonumber
    &\frac{\partial}{\partial\theta} \mathcal{D}_{FD}(\theta) = \frac{\partial}{\partial\theta}\mathbb{E}_{\xx_\sigma \sim p_{\theta, \sigma}} \|\bm{s}_{q}(\xx_\sigma) - \bm{s}_{\theta, \sigma}(\xx_\sigma)\|_2^2 \\
    \nonumber
    & = \mathbb{E}_{\xx_\sigma \sim p_{\sigma,\theta}} \bigg\{  \frac{\partial}{\partial \xx_\sigma}\big\{ \|\bm{s}_{q}(\xx_\sigma) - \bm{s}_{\theta, \sigma}(\xx_\sigma)\|_2^2 \big\} \frac{\partial \xx_\sigma(\theta)}{\partial\theta} \\ 
    \nonumber
    & - 2 \big[ \bm{s}_{q}(\xx_\sigma) - \bm{s}_{\theta,\sigma}(\xx_\sigma) \big]^T \frac{\partial}{\partial\theta} \bm{s}_{\theta,\sigma}(\xx_\sigma) \bigg\}\\
    & = \operatorname{Grad}_{1}(\theta) + \operatorname{Grad}_{1}(\theta).
\end{align}
Where $\operatorname{Grad}_{1}(\theta)$ and $\operatorname{Grad}_{2}(\theta)$ are defined with
\begin{align}
    \nonumber
    & \operatorname{Grad}_{1}(\theta) = \mathbb{E}_{p_{\sigma,\theta}} \bigg\{  \frac{\partial}{\partial \xx_\sigma}\big\{ \|\bm{s}_{q}(\xx_\sigma) - \bm{s}_{\theta, \sigma}(\xx_\sigma)\|_2^2 \big\} \frac{\partial \xx_\sigma(\theta)}{\partial\theta} \bigg\}, \\
    \nonumber
    & \operatorname{Grad}_{2}(\theta) = \mathbb{E}_{p_{\sigma,\theta}} \bigg\{  - 2 \big[ \bm{s}_{q}(\xx_\sigma) - \bm{s}_{\theta,\sigma}(\xx_\sigma) \big]^T \frac{\partial}{\partial\theta} \bm{s}_{\theta,\sigma}(\xx_\sigma) \bigg\}.
\end{align}
The gradient formula of \eqref{eqn:tFD_grad_full} considers all path derivatives concerning parameter $\theta$. We put a detailed derivation in the Appendix. 

Notice that the first gradient term $\operatorname{Grad}_{1}(\theta)$ can be obtained if we stop the $\theta$ gradient for $\bm{s}_{\theta,\sigma}(.)$, i.e. $\bm{s}_{\operatorname{sg}[\theta],\sigma}(.)$ and minimize a corresponding loss function 
\begin{align}\label{eqn:tFD_loss1}
    \mathcal{L}_{1}(\theta) &= \mathbb{E}_{\xx_\sigma \sim p_{\sigma,\theta}} \bigg\{ \|\bm{s}_{q}(\xx_\sigma) - \bm{s}_{\operatorname{sg}[\theta], \sigma}(\xx_\sigma)\|_2^2  \bigg\} \\
    \nonumber
    &= \mathbb{E}_{\zz\sim p_z,\epsilon\sim \mathcal{N}(\bm{0}, \mathcal{I}) \atop \xx_\sigma = g_\theta(\zz)+\sigma\epsilon} \bigg\{ \|\bm{s}_{q}(\xx_\sigma) - \bm{s}_{\operatorname{sg}[\theta], \sigma}(\xx_\sigma)\|_2^2 \bigg\}
\end{align}
However, the second gradient $\operatorname{Grad}_{2}(\theta)$ include the term $\frac{\partial}{\partial \theta} \bm{s}_{\theta, \sigma}(.)$ which is unknown yet intractable. This is because, for the implicit sampler, we only have efficient samples from the implicit distribution, but the score function $\bm{s}_{\theta, \sigma}(.)$ along with its $\theta$ derivative is unknown.

\paragraph{The Tractable yet Equivalent Objective.}
Though gradient $\operatorname{Grad}_{2}(\theta)$ is intractable, fortunately, in this paper, we can address such an issue by introducing theoretical tools in Theorem \ref{thm:score_derivative_identity}. 
\begin{theorem}\label{thm:score_derivative_identity}
    If distribution $p_{\theta, \sigma}$ satisfies some wild regularity conditions, then we have for all vector-valued score function $\bm{s}_{q}(.)$, the equation holds for all parameter $\theta$:
    \begin{align}
    \nonumber
    & \mathbb{E}_{\xx_\sigma \sim p_{\sigma,\theta}} \bigg\{  - 2 \big[ \bm{s}_{q}(\xx_\sigma) - \bm{s}_{\theta,\sigma}(\xx_\sigma) \big]^T \frac{\partial}{\partial\theta} \bm{s}_{\theta,\sigma}(\xx_\sigma) \bigg\} \\
    \nonumber
    &= \frac{\partial}{\partial\theta}  \mathbb{E}_{\zz\sim p_Z, \xx_0 = g_\theta(\zz), \atop \epsilon\sim \mathcal{N}(\bm{0},\mathcal{I}),\xx_\sigma = \xx_0 + \sigma\epsilon} \bigg\{2 \big[ \bm{s}_{q}(\xx_\sigma) - \bm{s}_{\operatorname{sg}[\theta],\sigma}(\xx_\sigma) \big]^T\\
    \label{eqn:tractable_loss_grad2}
    & \big[\bm{s}_{\operatorname{sg}[\theta], \sigma}(\xx_\sigma) - \nabla_{\xx_{\sigma}} \log p(\xx_\sigma|\xx_0) \big] \bigg\}
    \end{align}
\end{theorem}
We put the detailed proof in the Appendix. Theorem \ref{thm:score_derivative_identity} shows that the intractable gradient $\operatorname{Grad}_{2}(\theta)$ happens to be the same as the gradient of a tractable loss \eqref{eqn:tractable_loss_grad2} if we can have a good pointwise approximation of the score network $\bm{s}_{\operatorname{sg}[\theta], \sigma}(\cdot, \cdot)$. In practice, such an approximation is not difficult due to mature score function estimation techniques as we will discuss in Section \ref{sec:background}. Therefore, we can minimize a tractable loss function \eqref{eqn:tFD_loss2} using gradient-based algorithms to get the intractable gradient $\operatorname{Grad}_{2}(\theta)$:
\begin{align}\label{eqn:tFD_loss2}
    \nonumber
    \mathcal{L}_2(\theta) &= \mathbb{E}_{\zz\sim p_Z, \xx_0 = g_\theta(\zz), \atop \epsilon\sim \mathcal{N}(\bm{0},\mathcal{I}),\xx_\sigma = \xx_0 + \sigma\epsilon} \bigg\{2 \big[ \bm{s}_{q}(\xx_\sigma) - \bm{s}_{ \operatorname{sg}[\theta],\sigma}(\xx_\sigma) \big]^T \\
    \nonumber
    & \big[\bm{s}_{\operatorname{sg}[\theta], \sigma}(\xx_\sigma) - \nabla_{\xx_{\sigma}} \log p(\xx_\sigma|\xx_0) \big] \bigg\}\\
    \nonumber
    & = \mathbb{E}_{\zz\sim p_Z, \xx_0 = g_\theta(\zz), \atop \epsilon\sim \mathcal{N}(\bm{0},\mathcal{I}),\xx_\sigma = \xx_0 + \sigma\epsilon} \bigg\{2\big[\bm{s}_{q}(\xx_\sigma) - \bm{s}_{\operatorname{sg}[\theta],\sigma}(\xx_\sigma) \big]^T \\
    & \big[\bm{s}_{\operatorname{sg}[\theta], \sigma}(\xx_\sigma) - \nabla_{\xx_{\sigma}} \log p(\xx_\sigma|\xx_0) \big] \bigg\}
\end{align}
Combining with \eqref{eqn:tFD_loss1} and \eqref{eqn:tFD_loss2} with \eqref{eqn:tFD_grad_full}, we can have the final equivalent loss function that can minimize the Fisher divergence which writes 
\begin{align}\label{eqn:tFD_loss_full}
    \mathcal{L}_{DFT}(\theta) = \mathcal{L}_1(\theta) + \mathcal{L}_2(\theta)
\end{align}
Where $\mathcal{L}_1(\theta)$ and $\mathcal{L}_2(\theta)$ are defined in \eqref{eqn:tFD_loss1} and \eqref{eqn:tFD_loss2}. We name our training objective the denoising Fisher training objective because it originates from minimizing the Fisher divergence between the slightly tweaked sampler and the target distribution. 

\paragraph{The Practical Algorithm}
\begin{algorithm}[]
\SetAlgoLined
\KwIn{un-normalized target $\log q(\xx)$, latent distribution $p_z(\zz)$, implicit sampler $g_\theta(.)$, score network $\bm{s}_\phi(.)$, mini-batch size B, max iteration M.}
Randomly initialize $(\theta^{(0)}, \phi^{(0)})$.\\
\For{$t$ in 0:M}{
\emph{// update score network parameter}\\
Freeze $\theta$, free $\phi$\\
Get mini-batch without parameter dependence\\
$x_i = \operatorname{sg}[g_{\theta^{(t)}}(\zz_i)], \zz_i\sim p_z(\zz), i=1,..,B$.\\
Update Score network $\phi$ by minimizing \eqref{eqn:sm_objective} or \eqref{eqn:dsm_objective}, with samples drawn from implicit samplers and detach the $\theta$ gradient \\
\emph{// update sampler parameter}\\
Freeze $\phi$, free $\theta$\\
Get mini-batch differentiable samples with $\theta$-parameter-dependecne\\
$\xx_i = g_{\theta^{(t)}}(\zz_i), \zz_i\sim p_z(\zz), i=1,..,B$.\\
Calculate loss with \eqref{eqn:tFD_loss_full}\\
Update $\theta$ with gradient-based optimization algorithms to get $\theta^{(t+1)}$.}
\Return{$(\theta,\phi)$.}
\caption{The Denoising Fisher Training (DFT) Algorithm for Neural Implicit Samplers.}
\label{alg:denoise_training}
\end{algorithm}
Till now, we can formally define the denoising Fisher training algorithm for implicit samplers in the Algorithm \ref{alg:denoise_training}. As we can see, the algorithm involves two neural networks, with one $g_\theta(\cdot)$ being the implicit sampler and $\bm{s}_{\phi}(\cdot)$ being the neural score network to estimate the score function of the implicit sampler. The algorithm involves iterations that alternatively update each neural network while freezing the other one. In the first step of each iteration, we freeze the $\theta$ and we generate a batch of samples efficiently from the neural implicit sampler $g_\theta$. Then we use the generated samples to update the score network $\bm{s}_{\phi}$ to approximate the implicit sampler's score function. In the second step, we freeze $\phi$ and minimize \eqref{eqn:tFD_loss_full} to update $g_\theta$. 

\begin{figure}
\centering
\includegraphics[width=1.0\linewidth]{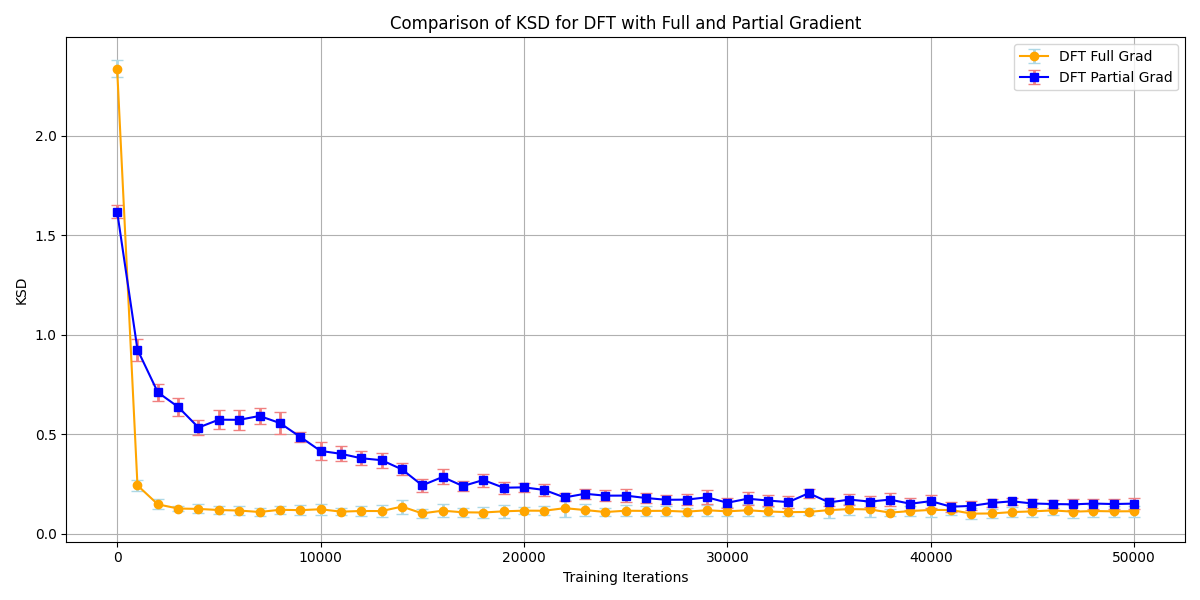}
\caption{The comparison of KSD values of DFT-NS using full and partial gradients in \eqref{eqn:tFD_grad_full}.}
\label{fig:dft_compare}
\end{figure}

\begin{figure*}
\centering
\includegraphics[width=1.0\linewidth]{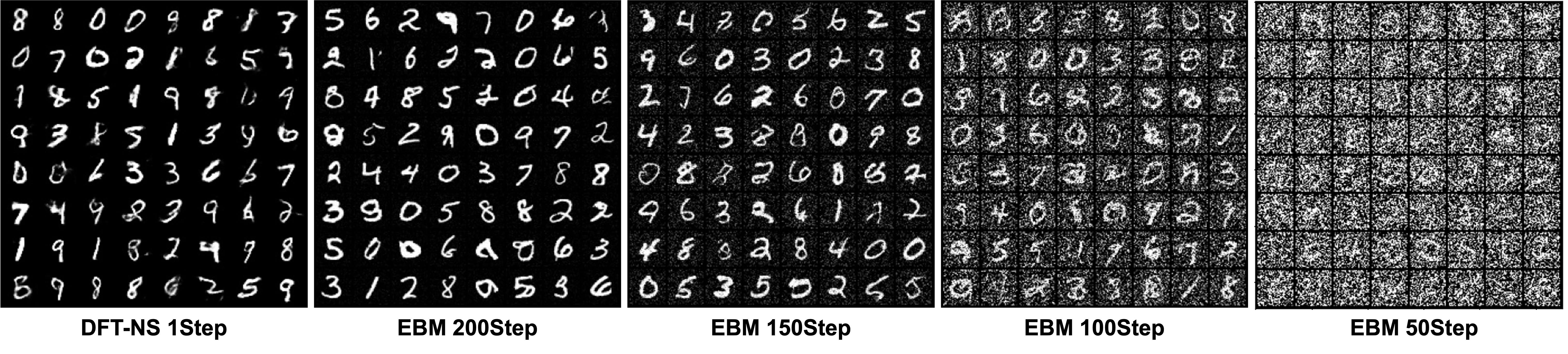}
\caption{A visualization of DFT-NS on MNIST again default MCMC sampling result from DeepEBM. The one-step DFT-NS outperforms MCMC samplers with 200 sampling steps as shown in Table \ref{tab:t5}.}
\label{fig:mnist_compare}
\end{figure*}
\section{EXPERIMENTS}\label{sec:exp}
In this section, we validate neural samplers trained with DFT across a range of three distinct sampling benchmarks following \citet{luo2024entropy}. These benchmarks span a spectrum of complexity, from low-dimensional (2-dimensional targets) to high-dimensional (784-dimensional image targets) tasks. 

\subsection{2D Synthetic Target Sampling}
\paragraph{Experiment Settings.}
We leverage the open-source implementation from \citet{sharrock2023coin}\footnote{\url{https://github.com/louissharrock/coin-svgd}} and adhere to the experimental configurations established by \citet{luo2024entropy} for the training of samplers on six 2D target distributions. Our comparison encompasses a range of methods, including 3 MCMC baselines: Stein variational gradient descent (SVGD) \citep{liu2016stein}, Langevin dynamics (LD) \citep{welling2011bayesian}, and Hamiltonian Monte Carlo (HMC) \citep{neal2011mcmc}; one explicit baseline: coupling normalizing flow \citep{Dinh2016DensityEU}; and four implicit samplers: KL training sampler \citep{luo2024entropy}, Fisher training sampler \citep{luo2024entropy}, KSD neural sampler (KSD-NS) \citep{nss}, and SteinGan \citep{wang2016learning}. All implicit samplers are designed with a uniform neural architecture, which consists of a four-layer multilayer perceptron (MLP) with 400 hidden units per layer, and employ ELU activation functions for both the sampler and the score network where applicable.

We evaluate sampling results with kernelized Stein's discrepancy (KSD) \citep{Liu2016AKS}, which is a widely used metric for evaluating the quality of a batch of samples to a target distribution with score functions \citep{Liu2016AKS,gorham2015measuring}. In our evaluations, we utilize the KSD with the inverse multiquadric (IMQ) kernel, which is implemented through the open-source package sgmcmcjax\footnote{\url{https://github.com/jeremiecoullon/SGMCMCJax}}. We put the quantitative results in Table \ref{tab:t1}.

\begin{table*}[h]
\caption{KSD values of samplers. For MCMC, we set the Stepsize=0.01, num particles=500, num chains=20.}
\label{tab:1}
\begin{center}
\begin{scriptsize}
\begin{sc}
\begin{tabular}{lccccccc}
\toprule
\textbf{Sampler} & \textbf{Gaussian} & \textbf{MOG2} & \textbf{Rosenbrock} & \textbf{Donut} & \textbf{Funnel} & \textbf{Squiggle} \\
\midrule
MCMC &  &  &  &  &  & \\
\midrule
SVGD(500) & $0.013\pm 0.001$ & $0.044\pm 0.006$ & $0.053\pm 0.002$ & $0.057\pm 0.004$ & $0.052\pm 0.001$ & $0.024\pm 0.002$ \\
SVGD(200) & $0.0367\pm 0.016$ & $0.0545\pm 0.006$ & $0.513\pm 0.040$ & $0.0834\pm 0.007$ & $0.0502\pm 0.002$ & $0.0391\pm 0.006$ \\
SVGD(100) & $0.776\pm 0.069$ & $0.688\pm 0.027$ & $1.498\pm 0.078$ & $0.361\pm 0.023$ & $0.189\pm 0.023$ & $0.156\pm 0.017$ \\
SVGD (50) & $1.681\pm 0.073$ & $1.170\pm 0.030$ & $2.673\pm 0.163$ & $1.089\pm 0.035$ & $0.612\pm 0.072$ & $0.682\pm 0.045$ \\
LD(500) & $0.107\pm 0.025$ & $0.099\pm 0.008$ & $0.152\pm 0.030$ & $0.107\pm 0.020$ & $0.116\pm 0.029	$ & $0.139\pm 0.030$ \\
LD(200) & $0.141\pm 0.024$ & $0.109\pm 0.020$ & $0.284\pm 0.037$ & $0.109\pm 0.021$ & $0.144\pm 0.017$ & $0.176\pm 0.057$ \\
LD(100) & $0.277\pm 0.042$ & $0.276\pm 0.018$ & $0.691\pm 0.030$ & $0.115\pm 0.024$ & $0.231\pm 0.029$ & $0.264\pm 0.055$ \\
LD (50) & $0.669\pm 0.073$ & $0.678\pm 0.016$ & $1.198\pm 0.038$ & $0.246\pm 0.037$ & $0.387\pm 0.033$ & $0.419\pm 0.046$ \\
HMC(500) & $0.094\pm 0.020$ & $0.106\pm 0.020$ & $0.134\pm 0.034$ & $0.113\pm 0.020$ & $0.135\pm 0.010$ & $0.135\pm 0.033$ \\
HMC(200) & $0.109\pm 0.025$ & $0.104\pm 0.018$ & $0.137\pm 0.023$ & $0.107\pm 0.016$ & $0.132\pm 0.028$ & $0.143\pm 0.024$ \\
HMC(100) & $0.114\pm 0.031$ & $0.106\pm 0.018$ & $0.205\pm 0.031$ & $0.115\pm 0.017$ & $0.129\pm 0.029$ & $0.159\pm 0.038$ \\
HMC (50) & $0.199\pm 0.038$ & $0.176\pm 0.033$ & $0.501\pm 0.032$ & $0.112\pm 0.019$ & $0.176\pm 0.033$ & $0.234\pm 0.046$ \\
\midrule
Neural Samplers &  &  &  &  &  & \\
\midrule
Coup-Flow & $0.102\pm 0.028$ & $0.158\pm 0.019$ & $0.150\pm 0.026	$ & $0.239\pm 0.013$ & $0.269\pm 0.019$ & $0.130\pm 0.026$ \\
KSD-NS & $0.206\pm 0.043$ & $1.129\pm 0.197$ & $1.531\pm 0.058$ & $0.341\pm 0.039$ & $0.396\pm 0.221$ & $0.462\pm 0.065$ \\
SteinGAN & $0.091\pm 0.013$ & $0.131\pm 0.011$ & $0.121\pm 0.022$ & $0.104\pm 0.013$ & $0.129\pm 0.020$ & $0.124\pm 0.018$ \\
{Fisher-NS} & $0.095\pm 0.016$ & $0.118\pm 0.013$ & $0.157\pm 0.030$ & $0.179\pm 0.028$ & $7.837\pm 1.614$ & $0.202\pm 0.037$ \\
{KL-NS} & $0.099\pm 0.015$ & $0.104\pm 0.015$ & $0.123\pm 0.021$ & $0.109\pm 0.015$ & $0.115\pm 0.012$ & $0.118\pm 0.024$ \\
\textbf{DFT-NS} & $0.098\pm 0.019$ & $0.101\pm 0.025$ & $0.121\pm 0.016$ & $0.100\pm 0.013$ & $0.081\pm 0.009$ & $0.115\pm 0.018$ \\
\bottomrule
\end{tabular}
\end{sc}
\end{scriptsize}
\end{center}
\label{tab:t1}
\end{table*}

Across all target distributions, we have trained all implicit neural samplers (DFT-NS in Table \ref{tab:1}) using the Adam optimizer. We use the learning rate of 2e-5 and a batch size of 5000 by default unless specifically claimed. To evaluate the kernelized Stein's discrepancy (KSD), we have conducted assessments every 1000 iterations, using 500 samples and repeating the process 20 times for each evaluation. The lowest mean KSD observed during training iterations has been selected as our final result. The neural architecture and evaluation setup is the same as \citet{luo2024entropy}.

\paragraph{Performance Analysis.}
As Table \ref{tab:t1} shows, 
The Stein variational gradient descent (SVGD) with 500 sampling steps shows the best KSD values. However, when the sampling steps of SVGD are limited to less than 200, the performance drops significantly. For relatively simple target distribution such as Gaussian distribution, the SteinGAN is the best. However, for complex targets such as Donut, Funnel, and Squiggle distributions, the DFT-NS is the best among all neural samplers. The visualization in Figure \ref{fig:big_vis} qualitatively confirms the quantitative observation. 

Notice in \eqref{eqn:tFD_grad_full}, we decompose the gradient into two terms. In Figure \ref{fig:dft_compare}, we compare the samplers' KSD that is trained using full gradient \eqref{eqn:tFD_grad_full} and solely the second gradient $\operatorname{Grad}_2(\theta)$ which we term the DFT partial grad on Funnel distribution. The results show that using a full gradient is consistently better than the partial gradient. However, as we can see, only using $\operatorname{Grad}_2(\theta)$ still leads to solid training. In our practical experiments, we find that sometimes balancing both gradients can result in better performances than using both or one of them.  

In summary, across both analytic and neural targets, DFT-NS demonstrate a marked advantage in efficiency while maintaining comparable or better performance than MCMC samplers. This efficiency makes DFT-NS a more viable option for a wide array of sampling tasks where high efficiency and reduced computational costs are of the essence.

\begin{figure}
\centering
\includegraphics[width=1.0\linewidth]{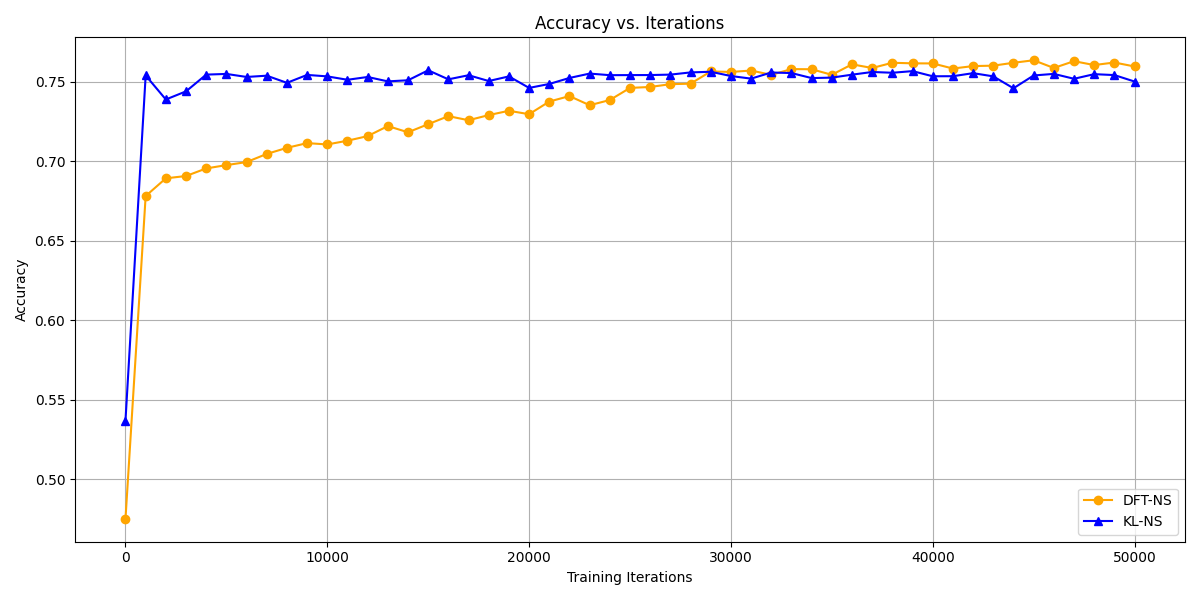}
\caption{The test-accuracy curve of DFT-NS (ours) and KL-NS \citep{luo2024entropy} for Bayesian Logistic Regression. Though KL-NS converges faster, DFT-NS shows better final accuracy than KL-NS with sufficient training iterations.} 
\label{fig:bayes_compare}
\end{figure}

\subsection{Bayesian Logistic Regression}
Having established the effectiveness of our sampler on low-dimensional 2D target distributions in the previous section, we now turn our attention to more complex, real-world scenarios. Bayesian logistic regression offers a set of medium-dimensional target distributions (ranging from 10 to 100 dimensions) that are ideal for this purpose. In this section, we pit our DFT neural samplers on solving this problem following \citet{nss} and \citet{luo2024entropy}. 

The primary objective of this experiment is to assess the performance of our proposed sampler in terms of test accuracy when applied to Bayesian logistic regression tasks. To provide a comprehensive evaluation, we also compare our sampler's performance against that of various MCMC samplers. This comparison will shed light on how our sampler stacks up against existing methods in a real-world application, allowing us to evaluate its practical utility and effectiveness in handling medium-dimensional target distributions.

\paragraph{Experiment settings.} 
We have adopted the experimental setup similar to that of \citet{nss} for the Bayesian logistic regression tasks. The Covertype dataset, which comprises 54 features and 581,012 observations, is a well-established benchmark for Bayesian inference. In line with \citet{nss}, we have specified the prior distribution for the weights as \(p(w|\alpha) = \mathcal{N}(w;~0, \alpha^{-1})\) and for \(\alpha\) as \(p(\alpha) = \text{Gamma}(\alpha;~ 1, 0.01)\). The dataset has been randomly divided into a training set, accounting for 80

Our objective is to develop an implicit DFT neural sampler that can efficiently sample from the posterior distribution. Once trained, the implicit sampler not only delivers the best accuracy but also operates hundreds of times faster than MCMC methods. In this experiment, we train DFT-NS and KL-NS using the same configuration as proposed in \citet{luo2024entropy}. Our evaluation metric is the test accuracy, which marks how well samplers can solve the Bayesian logistic regression. 

\paragraph{Performance Analysis.}
Table \ref{Covertype} shows the quantitative performances of different samplers. As we can see, the DFT with full loss in \eqref{eqn:tFD_loss_full} results in the best test accuracy of $76.36\%$, outperforming KL-NS \citet{luo2024entropy}, Fisher-NS \citep{nss}, and MCMC algorithms. The DFT-NS with only the second loss of \eqref{eqn:tFD_loss2} also shows a decent performance. 

We also conduct a detailed comparison of DFT-NS with KL-NS in Figure \ref{fig:bayes_compare} by recording the test accuracy of both samplers along the training process. The results show that KL-NS converges faster than DFT-NS, however, DFT-NS archives a better final test accuracy with sufficient training iterations. These experiments underscore the capability of our proposed training method to effectively manage real-world Bayesian inference tasks involving medium-dimensional data. 

\begin{table}[h]
\caption{Test Accuracies for Bayesian Logistic Regression on Covertype Dataset following \citet{nss}. DFT-NS is the neural implicit sampler trained with DFT with full loss in \eqref{eqn:tFD_loss_full}. DFT-NS$^*$ means we only use the second loss \eqref{eqn:tFD_loss2} for training.} 
\label{Covertype}
\begin{center}
\begin{small}
\begin{sc}
\begin{tabular}{cc}
\toprule
SGLD & DSVI \\
75.09\% $\pm$ 0.20\% & 73.46\% $\pm$ 4.52\%  \\
\midrule
SVGD & SteinGAN \\
74.76\% $\pm$ 0.47\% & 75.37\% $\pm$ 0.19\% \\ 
\midrule
Fisher-NS & {KL-IS(Ours)} \\
76.22\% $\pm$ 0.43\% & 75.95\% $\pm$ 0.002\% \\
\midrule
\textbf{DFT-NS(Ours)}$^*$ & \textbf{DFT-NS(Ours)} \\
75.79\% $\pm$ 0.55\% & 76.36\% $\pm$ 0.009\% \\
\bottomrule
\end{tabular}
\end{sc}
\end{small}
\end{center}
\end{table}

\begin{figure}
\centering
\includegraphics[width=1.0\linewidth]{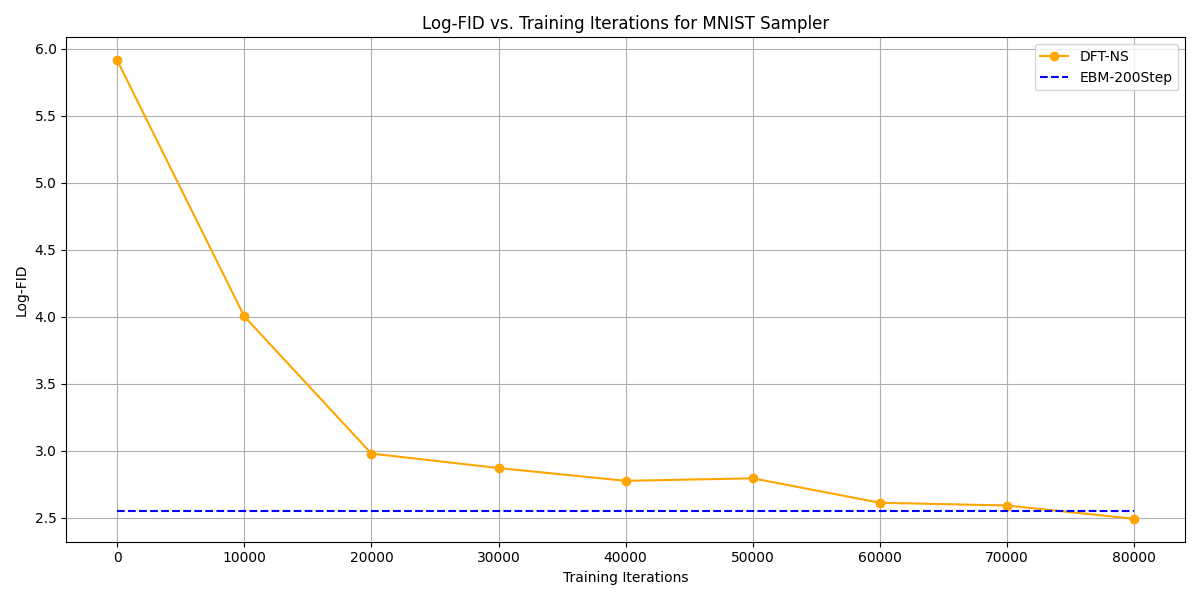}
\caption{The Log-FID curve of one-step DFT-NS on MNIST image distribution. The blue dashed line marks the logarithm of FID of 200-step MCMC (Annealed Langevin dynamics) samples from pre-trained DeepEBM \citep{Li2019LearningEM}.}
\label{fig:mnist_fid}
\end{figure}

\subsection{Sampling from Energy-based Models}
\begin{table}
\scriptsize
\centering
\caption{Comparison of Sampling Efficiency and FID values of multi-step MCMC samplers and one-step DFT neural sampler. The larger the FLOPS is, the more costs the sampler has.}
\begin{center}
\begin{tabular}{l|ccccc}
\toprule
Model & Steps & Infer Time (Sec) & FLOPS & FID \\
\midrule
EBM   & 250 & 0.8455 & 0.58G$\times$250 & 9.54 \\
EBM   & 200 & 0.6658 & 0.58G$\times$200 & 12.76 \\
EBM   & 150 & 0.5115 & 0.58G$\times$150 & 165.07 \\
EBM   & 100 & 0.3455 & 0.58G$\times$100 & 423.95 \\
EBM   & 50 & 0.1692 & 0.58G$\times$50 & 519.04 \\
\textbf{DFT-NS} & {1} & {0.0012} & {1.11G} & 12.07 \\
\bottomrule
\end{tabular}
\end{center}
\label{tab:t5}
\vspace{-5mm}
\end{table}

A pivotal benefit of employing an implicit sampler is its superior inference efficiency, which is particularly advantageous for applications where traditional sampling methods fall short, such as in the context of high-dimensional energy-based models (EBMs). The inefficiency of default sampling methods in these scenarios often stems from the complex and intricate nature of the models, which can make sampling a computationally demanding task.

\paragraph{Experiment settings.}
A neural network in an energy-based model (EBM) \citep{LeCun2006ATO} defines a negative energy function, essentially the log of an un-normalized probability distribution. Post-training, extracting samples from an EBM typically involves an annealed Markov Chain Monte Carlo process \citep{Xie2016ATO}, which is quite computationally demanding. In this experiment, we follow \citet{luo2024entropy} to evaluate the performance of DFT-NS to learn to sample from a deep EBM \citep{Li2019LearningEM} pre-trained on the MNIST image dataset. Considering the characteristics of image data, we closely follow \citet{luo2024entropy} to construct the neural sampler by stacking multiple convolutional neural network blocks. It processes a random Gaussian input of 128 dimensions to produce a 32x32 tensor, matching the resolution of the EBM's training. We pre-train the DeepEBM in-house following \citet{Li2019LearningEM}. Then we train the CNN-based implicit neural sampler with DFT using the un-normalized probability induced by the pre-trained DeepEBM. We follow the training settings of \citet{luo2024entropy}, using a batch size of 128 and Adam optimizer with a learning rate of 1e-4 on a single Nvidia-A100 GPU. The hours for one training trial is roughly three hours, which is conceptually cheap. 

To quantitatively evaluate DFT-NS and MCMC samplers, we follow \citet{luo2024entropy,heusel2017gans} to train an MNIST classifier and compute the Frechet Inception Distance (FID) \citep{heusel2017gans} using features extracted from the pre-trained classifier. We constructed the classifier with the wide-resnet neural network \citep{zagoruyko2016wide}. We compare the FID of samples generated from one-step DFT-NS and multi-step MCMCs with the DeepEBM. The smaller the FID is, the better the sample quality is.

\paragraph{Performance.}
Table \ref{tab:t5} shows the FID values and corresponding computational costs of each sampler. DFT-NS produces samples with better FID values than MCMC samplers with 200 sampling steps. But it is worse than MCMC samples with 250 sampling steps. This indicates that DFT-NS is better than MCMC with an acceleration rate larger than 200. Figure \ref{fig:mnist_compare} shows a visualization of samples drawn from DFT-NS and MCMC samplers. This visualization qualitatively confirms the quantitative results in Table \ref{tab:t5}. 

In Figure \ref{fig:mnist_compare}, we visualize the FID curve of the one-step DFT-NS versus the training iterations. We can conclude that DFT is a solid training approach that trains the neural implicit samplers to have a steady performance curve with the increase of training iterations. 

\section{CONCLUSION AND LIMITATIONS}
In this paper, we've introduced denoising implicit training (DFT), an innovative technique for training implicit samplers to draw samples from distributions with un-normalized density. We've shown theoretically that DFT is equivalent to minimizing the Fisher divergence between a slightly tweaked neural sampler distribution and the target distribution. Through intensive quantitative experiments on scales of small, medium, and high dimensions, we show that neural samplers trained with DFT have better performances than other neural samplers and MCMC samplers with a decent number of sampling steps. Besides, our theoretical assessments on deriving the tractable loss may shed light on broader applications when handling intractable Fisher divergences. 

Nevertheless, there are certain limitations to our proposed methods. First, the process of score estimation is not computationally inexpensive, making it a crucial avenue for research to develop a more efficient training algorithm that could eliminate the score estimation phase. Furthermore, currently, our sampler is confined to sampling tasks. Exploring how to extend this methodology to other applications, such as generative modeling, presents another intriguing area for future investigation.


\clearpage
\newpage
\bibliography{bib}
\bibliographystyle{IEEEtranN}

\onecolumn

\appendix

\section{Pseudo Code for DFT Algorithm for Training Samplers from EBM}
\begin{algorithm}[H]
\caption{Pytorch Style Pseudo Python Code for the DFT Training Method}
\DontPrintSemicolon
\SetAlgoLined
\SetKwInOut{Input}{input}\SetKwInOut{Output}{output}
\Input{Generator $G$, Discriminator $D$, Diffusion process $diffusion$, Energy function $energy\_fun$, Maximum iterations $max\_iter$, Batch size $batch\_size$, Optimizer arguments for $D$ and $G$, Annealing flag $anneal$, Validation interval $val\_interval$, Number of discriminator steps $D\_steps$, Distillation method $distill\_method$, Prefix $prefix$}
\Output{Trained generator $G$ and discriminator $D$}

\BlankLine
$device \gets 'cuda'$\;
$Doptim \gets Adam(D.parameters(), **D\_opt\_args)$\;
$Goptim \gets Adam(G.parameters(), **G\_opt\_args)$\;

\For{$iiter \in \{0, 1, \ldots, max\_iter\}$}{
    \If{$distill\_method == 'dft-full'$}{
        $D.train().requires\_grad\_(True)$\;
        $G.eval().requires\_grad\_(False)$\;
        \For{$\_ \in \{1, 2, \ldots, D\_steps\}$}{
            $z \gets \text{torch.randn}((batch\_size, G.latent\_dim, 1,1)).cuda()$\;
            $fake\_x \gets G(z)$\;
            $t \gets \text{random choice from } \{1, 2, \ldots, diffusion.T\}$\;
            $fake\_t, t, noise, sigma\_t, g2\_t \gets diffusion(fake\_x, t=t, return\_t=True)$\;
            $wgt \gets g2\_t/sigma\_t.square()$\;
            $d\_loss \gets 0.5 \times wgt \times (D(fake\_t, t) + noise).square().sum([1,2,3])$\;
            $d\_loss \gets d\_loss.mean()$\;
            $Doptim.zero\_grad()$\;
            $d\_loss.backward()$\;
            $Doptim.step()$\;
        }
        Append $d\_loss.item()$ to $dlosses$\;
        $D.eval().requires\_grad\_(False)$\;
        $G.train().requires\_grad\_(True)$\;
        $z \gets \text{torch.randn}((batch\_size, G.latent\_dim, 1,1)).to(device)$\;
        $fake\_x \gets G(z)$\;
        $t \gets \text{random choice from } \{7, 8, \ldots, 50\}$\;
        $fake\_t, t, noise, sigma\_t, g2\_t \gets diffusion(fake\_x, t=t, return\_t=True)$\;
        $sigma\_t \gets sigma\_t.view(-1,1,1,1)$\;
        $_fake\_t \gets fake\_t.clone().detach().requires\_grad\_(True)$\;
        $lam \gets 1.0$\;
        $score\_true \gets \text{torch.autograd.grad}(energy\_fun(_fake\_t,lam=lam).sum(), _fake\_t, create\_graph=True, retain\_graph=True)[0]/sigma\_t$\;
        $score\_fake \gets D(_fake\_t, t)/sigma\_t$\;
        $score\_diff \gets score\_fake - score\_true$\;
        $cond\_score \gets -noise/sigma\_t$\;
        $dloss\_dx \gets -2 \times (score\_diff \times (score\_fake - cond\_score)).sum([1,2,3])$\;
        $dloss\_dx \gets sigma\_t**2 \times \text{torch.autograd.grad}(dloss\_dx.sum(), _fake\_t, create\_graph=False, retain\_graph=False)[0]$\;
        $g\_loss \gets g2\_t \times (dloss\_dx \times fake\_t).sum([1,2,3])$\;
        $g\_loss \gets g\_loss.mean()$\;
        $Goptim.zero\_grad()$\;
        $g\_loss.backward()$\;
        $Goptim.step()$\;
        Append $g\_loss.item()$ to $glosses$\;
    }
    \If{$iiter \mod val\_interval == 0$}{
        \text{Save and plot intermediate results}\;
    }
}
\end{algorithm}

\section{Theory}

\subsection{Proof of \eqref{eqn:tFD_grad_full}}
\begin{proof}
\begin{align}\label{eqn:tFD_grad_full_app}
    \nonumber
    \frac{\partial}{\partial\theta} \mathcal{D}_{FD}(\theta) &= \frac{\partial}{\partial\theta}\mathbb{E}_{ \xx_\sigma\sim p_{\theta, \sigma}} \|\bm{s}_{q}(\xx_\sigma) - \bm{s}_{\theta, \sigma}(\xx_\sigma)\|_2^2 \\
    \nonumber
    & = \frac{\partial}{\partial\theta} \mathbb{E}_{\epsilon\sim \mathcal{N}(0,\mathbf{I}), \atop \zz\sim p_z} \|\bm{s}_{q}(g_\theta(\zz)+\sigma\epsilon) - \bm{s}_{\theta, \sigma}(g_\theta(\zz)+\sigma\epsilon)\|_2^2 \\
    & = \mathbb{E}_{\epsilon\sim \mathcal{N}(0,\mathbf{I}), \atop \zz\sim p_z}\bigg\{ \frac{\partial}{\partial } \big\{ \|\bm{s}_{q}(=g_\theta(\zz)+ \sigma\epsilon) - \bm{s}_{\theta, \sigma}(=g_\theta(\zz)+ \sigma\epsilon)\|_2^2 \big\}\frac{\partial g_\theta(\zz)+ \sigma\epsilon}{\partial\theta} \\
    & - 2 \big[ \bm{s}_{q}(\xx_\sigma) - \bm{s}_{\theta,\sigma}(\xx_\sigma) \big]^T \frac{\partial}{\partial\theta} \bm{s}_{\theta,\sigma}(\xx_\sigma)|_{\xx_\sigma=g_\theta(\zz)+ \sigma\epsilon} \bigg\}\\
    \nonumber
    & = \mathbb{E}_{\xx_\sigma\sim p_{\sigma,\theta}} \bigg\{  \frac{\partial}{\partial }\big\{ \|\bm{s}_{q}(\xx_\sigma) - \bm{s}_{\theta, \sigma}(\xx_\sigma)\|_2^2 \big\} \frac{\partial (\theta)}{\partial\theta} - 2 \big[ \bm{s}_{q}(\xx_\sigma) - \bm{s}_{\theta,\sigma}(\xx_\sigma) \big]^T \frac{\partial}{\partial\theta} \bm{s}_{\theta,\sigma}(\xx_\sigma) \bigg\}\\
    & = \operatorname{Grad}_{1}(\theta) + \operatorname{Grad}_{1}(\theta).
\end{align}
\end{proof}

\subsection{Proof of Theorem \ref{thm:score_derivative_identity}}

\begin{proof}
First, we prove a Lemma \ref{lemma:score_projectioin}. This lemma has been used for proving the equivalence of Denoising Score Matching \citep{vincent2011connection,zhou2024score} and Standard Score Matching \citep{hyvarinen2005estimation} as we introduced in Section \ref{sec:background}.

\begin{lemma}\label{lemma:score_projectioin}
    Let $\bm{u}(\cdot)$ be a vector-valued function, using the notations of Theorem \ref{thm:score_derivative_identity}. Let $q_\sigma(\xx_\sigma|\xx_0)=\mathcal{N}(\xx_\sigma; \xx_0, \sigma^2\mathcal{I})$, then under mild conditions, the identity holds:
    \begin{align}\label{eqn:lemma:score_projectioin}
        \mathbb{E}_{\xx_0\sim p_{0,\theta} \atop \xx_{\sigma}|\xx_0 \sim q_\sigma(\xx_{\sigma}|\xx_0)} \bm{u}(\xx_{\sigma},\theta)^T \bigg\{ \bm{s}_{\theta,\sigma}(\xx_{\sigma}) - \nabla_{\xx_{t}} \log q_{\sigma}(\xx_{\sigma}|\xx_0) \bigg\} = 0, ~~~ \forall \theta.
    \end{align}
\end{lemma}
\begin{proof}[Proof of Lemma \ref{lemma:score_projectioin}]

Recall the definition of $p_{\sigma,\theta}$ and $\bm{s}_{\theta, \sigma}$:
\begin{align}
    p_{\sigma,\theta}(\xx_{\sigma}) &= \int q_{\sigma}(\xx_{\sigma} | \xx_0) p_{\theta, 0}(\xx_0) \diff \xx_0 \\
    \bm{s}_{\theta, \sigma}(\xx_{\sigma}) &= \int \nabla_{\xx_{\sigma}} \log q_{\sigma}(\xx_{\sigma} | \xx_0) \frac{q_{\sigma}(\xx_{\sigma} | \xx_0) p_{0,\theta}(\xx_0)}{p_{\sigma,\theta}(\xx_{\sigma})} \diff \xx_0.
\end{align}
We may use $\bm{u}$ for short of $\bm{u}(\xx_{\sigma},\theta)$. We have 
\begin{align}
    \mathbb{E}_{\xx_{\sigma} \sim p_{\sigma,\theta}} \bm{u}^T \bm{s}_{\theta, \sigma}(\xx_{\sigma}) &= \mathbb{E}_{\xx_{\sigma} \sim p_{\sigma,\theta}} \bm{u}^T \int \nabla_{\xx_{\sigma}} \log q_{\sigma}(\xx_{\sigma} | \xx_0) \frac{q_{\sigma}(\xx_{\sigma} | \xx_0) p_{0,\theta}(\xx_0)}{p_{\sigma,\theta}(\xx_{\sigma})} \diff \xx_0\\
    & = \int p_{\sigma,\theta}(\xx_{\sigma}) \bm{u}^T \int \nabla_{\xx_{\sigma}} \log q_{\sigma}(\xx_{\sigma} | \xx_0) \frac{q_{\sigma}(\xx_{\sigma} | \xx_0) p_{0,\theta}(\xx_0)}{p_{\sigma,\theta}(\xx_{\sigma})} \diff \xx_0 \diff \xx_{\sigma} \\
    & = \int \int \bm{u}^T \nabla_{\xx_{\sigma}} \log q_{\sigma}(\xx_{\sigma} | \xx_0) q_{\sigma}(\xx_{\sigma} | \xx_0) p_{0,\theta}(\xx_0) \diff \xx_0 \diff \xx_{\sigma} \\
    & = \mathbb{E}_{\xx_0\sim p_{0,\theta},\atop \xx_{\sigma}|\xx_0 \sim q_{\sigma}(\xx_{\sigma}|\xx_0)} \bm{u}^T \nabla_{\xx_{\sigma}} \log q_{\sigma}(\xx_{\sigma} | \xx_0) 
\end{align}
\end{proof}

Next, we turn to prove the Theorem \ref{thm:score_derivative_identity}. Taking $\theta$ gradient on both sides of identity \eqref{eqn:lemma:score_projectioin}, we have
\begin{align}\label{eqn:derive_flow_product}
    & \mathbb{E}_{\xx_{\sigma}\sim p_{\theta,\sigma}} \bigg\{ \frac{\partial}{\partial\theta}\bm{u}(\xx_{\sigma}, \theta)^T \bm{s}_{\sigma, \theta}(\xx_{\sigma}) + \bm{u}(\xx_{\sigma}, \theta)^T \frac{\partial}{\partial\theta} \bm{s}_{\sigma, \theta}(\xx_{\sigma})\bigg\} + \mathbb{E}_{\xx_{\sigma}\sim p_{\theta,\sigma}} \frac{\partial}{\partial\xx_{\sigma}} \bigg\{ \bm{u}(\xx_{\sigma}, \theta)^T \bm{s}_{\sigma, \theta}(\xx_{\sigma}) \bigg \}\frac{\partial\xx_{\sigma}}{\partial\theta} \\\nonumber
    &= \mathbb{E}_{\xx_0\sim p_{0,\theta}, \atop \xx_{\sigma}|\xx_0 \sim q_{\sigma}(\xx_{\sigma}|\xx_0)} \frac{\partial}{\partial\theta}\bm{u}(\xx_{\sigma}, \theta)^T \nabla_{\xx_\sigma}\log q_{\sigma}  (\xx_{\sigma}|\xx_0) + \mathbb{E}_{\xx_0\sim p_{0,\theta}, \atop \xx_{\sigma}|\xx_0 \sim q_{\sigma}(\xx_{\sigma}|\xx_0)} \bigg\{\frac{\partial}{\partial\xx_{\sigma}}\bigg[ \bm{u}(\xx_{\sigma}, \theta)^T \nabla_{\xx_\sigma}\log q_{\sigma}  (\xx_{\sigma}|\xx_0)\bigg] \frac{\partial\xx_{\sigma}}{\partial\theta} \\\nonumber
    & + \bm{u}(\xx_{\sigma}, \theta)^T\frac{\partial}{\partial\xx_0}\nabla_{\xx_\sigma}\log q_{\sigma}  (\xx_{\sigma}|\xx_0)\frac{\partial\xx_0}{\partial\theta}\bigg\}
\end{align}

Notice that one can have 
$$
\mathbb{E}_{\xx_{\sigma}\sim p_{\theta,\sigma}} \bigg\{ \frac{\partial}{\partial\theta}\bm{u}(\xx_{\sigma}, \theta)^T\bigg\} \bm{s}_{\sigma, \theta}(\xx_{\sigma}) = \mathbb{E}_{\xx_0\sim p_{0,\theta}, \atop \xx_{\sigma}|\xx_0 \sim q_{\sigma}(\xx_{\sigma}|\xx_0)} \frac{\partial}{\partial\theta}\bm{u}(\xx_{\sigma}, \theta)^T \nabla_{\xx_\sigma}\log q_{\sigma}  (\xx_{\sigma}|\xx_0)$$ 
by substituting $ \bm{u}(\xx_{\sigma}, \theta)$ with $\frac{\partial}{\partial\theta}\bm{u}(\xx_{\sigma}, \theta)$ in equation (\ref{eqn:lemma:score_projectioin}).  

This allows us to cancel out the corresponding terms from equation (\ref{eqn:derive_flow_product}), and we have 
\begin{align}
    &\mathbb{E}_{\xx_{\sigma}\sim p_{\theta,\sigma}} \bigg\{ \bm{u}(\xx_{\sigma}, \theta)^T \frac{\partial}{\partial\theta} \bm{s}_{\sigma, \theta}(\xx_{\sigma})\bigg\} + \mathbb{E}_{\xx_{\sigma}\sim p_{\theta,\sigma}} \frac{\partial}{\partial\xx_{\sigma}} \bigg\{ \bm{u}(\xx_{\sigma}, \theta)^T \bm{s}_{\sigma, \theta}(\xx_{\sigma}) \bigg \}\frac{\partial\xx_{\sigma}}{\partial\theta}\\\nonumber
    &= \mathbb{E}_{\xx_0\sim p_{0,\theta}, \atop \xx_{\sigma}|\xx_0 \sim q_{\sigma}(\xx_{\sigma}|\xx_0)} \bigg\{\frac{\partial}{\partial\xx_{\sigma}}\bigg[ \bm{u}(\xx_{\sigma}, \theta)^T \nabla_{\xx_\sigma}\log q_{\sigma}  (\xx_{\sigma}|\xx_0)\bigg] \frac{\partial\xx_{\sigma}}{\partial\theta} + \bm{u}(\xx_{\sigma}, \theta)^T\frac{\partial}{\partial\xx_0}\nabla_{\xx_\sigma}\log q_{\sigma}  (\xx_{\sigma}|\xx_0)\frac{\partial\xx_0}{\partial\theta} \bigg\}
\end{align}

This gives rise to 
\begin{align}\label{eqn:flow_product_result}
    &\mathbb{E}_{\xx_{\sigma}\sim p_{\theta,\sigma}} \bigg\{ \bm{u}(\xx_{\sigma}, \theta)^T \frac{\partial}{\partial\theta} \bm{s}_{\sigma, \theta}(\xx_{\sigma})\bigg\} \\
    &= \mathbb{E}_{\xx_0\sim p_{0,\theta}, \atop \xx_{\sigma}|\xx_0 \sim q_{\sigma}(\xx_{\sigma}|\xx_0)} \bigg\{\frac{\partial}{\partial\xx_{\sigma}}\bigg[ \bm{u}(\xx_{\sigma}, \theta)^T \big\{ \nabla_{\xx_\sigma}\log q_{\sigma}  (\xx_{\sigma}|\xx_0) -\bm{s}_{\theta,\sigma}(\xx_\sigma) \big\}\bigg] \frac{\partial\xx_{\sigma}}{\partial\theta} + \bm{u}(\xx_{\sigma}, \theta)^T\frac{\partial\nabla_{\xx_\sigma}\log q_{\sigma}  (\xx_{\sigma}|\xx_0)}{\partial\xx_0}\frac{\partial\xx_0}{\partial\theta} \bigg\}\nonumber
\end{align}

We now define the following loss function 
\begin{align}
    \mathcal{L}_{2}(\theta) = \mathbb{E}_{\xx_0\sim p_{0,\theta}, \atop \xx_{\sigma}|\xx_0 \sim q_{\sigma}(\xx_{\sigma}|\xx_0)} \bigg\{ \bm{u}(\xx_{\sigma},\operatorname{sg}[\theta])^T \big\{ \nabla_{\xx_\sigma}\log q_{\sigma}  (\xx_{\sigma}|\xx_0) -\bm{s}_{\operatorname{sg}[\theta],\sigma}(\xx_{\sigma}) \big\}\bigg\} 
\end{align}
with $\bm{u}(\xx_{\sigma}, \theta) = -2\big\{ \nabla_{\xx_\sigma}\log q_{\sigma} (\xx_{\sigma}) - \bm{s}_{\sigma, \theta}(\xx_{\sigma}) \big\}$. Its gradient becomes 
\begin{align}
\nonumber
& \mathbb{E}_{\xx_{\sigma} \sim p_{\theta,\sigma}} \bigg\{  - 2 \big\{ \nabla_{\xx_\sigma}\log q_{\sigma} (\xx_{\sigma}) - \bm{s}_{\sigma, \theta}(\xx_{\sigma}) \big\}^T \frac{\partial}{\partial\theta} \bm{s}_{\sigma, \theta}(\xx_{\sigma}) \bigg\} \\
& = \frac{\partial}{\partial\theta}\mathbb{E}_{\xx_0\sim p_{0,\theta}, \atop \xx_{\sigma}|\xx_0 \sim q_{\sigma}(\xx_{\sigma}|\xx_0)} \bigg\{ \bm{u}(\xx_{\sigma},\operatorname{sg}[\theta])^T \big\{ \nabla_{\xx_\sigma}\log q_{\sigma}  (\xx_{\sigma}|\xx_0) -\bm{s}_{\operatorname{sg}[\theta]}(\xx_{\sigma},t) \big\}\bigg\} 
\end{align}
by applying the above result in (\ref{eqn:flow_product_result}). 
\end{proof}

\section{Experiments}\label{app:exp}

\subsection{Experiment Details on 2D Synthetic Sampling}
\paragraph{Model architectures.}
In our 2D synthetic data experiments, a 4-layer Multi-Layer Perceptron (MLP) neural network, equipped with 200 neurons in each layer, serves as our sampler. This network employs a LeakyReLU activation function with a leakage parameter of 0.2. Concurrently, the score network is also an MLP with 4 layers and 200 neurons per layer, utilizing GELU activation.

\paragraph{Comparison of DFT and KL training methods.}
In this work, we propose the DFT training method for neural implicit samplers, which is built on minimizing the Fisher divergence between slightly perturbed neural sampler distribution and the un-normalized target distribution. The DFT has similarities to previous methods such as KL training \citep{luo2024entropy}, which also has a paradigm of alternative training an online score network, and neural implicit samplers. However, DFT is different from KL training. In 2D target distribution sampling and the Bayesian logistic regression experiment, we empirically find that DFT training shows better performances than KL training. The reason behind such a plausible performance of the DFT sampler might be the different divergence that DFT and KL samplers are using. Besides, the small perturbation to sampler distributions also facilitates the distribution density of the DFT sampler to spread in all space, which potentially stabilizes the training dynamics.

\subsection{Experiment Details on Bayesian Regression}
\paragraph{Experiment settings.}
Following the setup in \cite{nss}, we define the prior of the weights as $p(w \mid \alpha)=\mathcal{N}\left(w ; 0, \alpha^{-1}\right)$ and $p(\alpha)=\operatorname{Gamma}(\alpha ; 1,0.01)$. The Covertype dataset \cite{Blackard1999ComparativeAO}, comprising 581,012 samples with 54 features, is split into a training set (80\%) and a testing set (20\%). Our methods are benchmarked against Stein GAN, SVGD, SGLD, DSVI \cite{Titsias2014DoublySV}, KSD-NS, and Fisher-NS \cite{nss}. While SGLD, DSVI, and SVGD are trained for 3 epochs (approximately 15k iterations), neural network-based methods (Fisher-NS, KSD-NS, Stein GAN, KL-NS, and DFT-NS) are trained until convergence. Given the equivalence of Fisher-NS \cite{nss}, we train the implicit sampler using the denoising Fisher training method. The learning rate for both the score network and the sampler is set to $0.0002$. The target distribution is approximated using 500 random data samples, with the score estimation phase set to 2. The logistic regression model's test accuracy is evaluated using 100 samples from the trained sampler.

For Bayesian inference, we adopt the same configuration as in \cite{nss}. The neural samplers are constructed with a 4-layer MLP containing 1024 hidden units per layer and GELU activations. The sampler's output dimension is 55, with an input dimension of 550, aligning with the setup in \citet{nss}. The score network mirrors the sampler's architecture but with an input dimension of 55. Adam optimizers with a learning rate of 0.0002 and default beta values are used for both networks. A batch size of 100 is employed for training the sampler, with the score network updated twice for every sampler update. Standard score matching is used to train the score network. The sampler undergoes 10k iterations per repetition, with 30 independent repetitions to determine the mean and standard deviation of the test accuracy. For SGLD, the learning rate follows $0.1/(t+1)^{0.55}$ as recommended in \cite{welling2011bayesian}, with the last 100 points averaged for evaluation. DSVI uses a learning rate of $1e-07$ with 100 iterations per stage. SVGD employs an RBF kernel with bandwidth determined by the median trick \cite{liu2016stein}, utilizing 100 particles with a step size of 0.05 for evaluation.

\paragraph{Model architectures.}
Our model employs 4-layer MLP neural networks for both the sampler and the score network. The GELU \cite{Hendrycks2016GaussianEL} activation function is utilized for both, with a hidden dimension of 1024. The sampler's input dimension is set to 128.

\subsection{Experiment Details on Sampling from EBMs}
\paragraph{Datasets and model architectures.} 
We pre-train a deep (multi-scale) Energy-Based Model (EBM) \cite{Li2019LearningEM}, denoted as $E_d(.)$, with 12 residual layers \cite{He2015DeepRL}, following the approach in DeepEBM \citep{Li2019LearningEM}. The noise levels are initialized at $\sigma_{min}=0.3$ and capped at $\sigma_{max}=3.0$. The energy for varying noise levels is modeled as $E_\sigma(\xx) = f(\xx)/\sigma$. With a learning rate of 0.001, the EBM is pre-trained for 200k iterations. Samples are extracted from the deep EBM using an annealed Langevin dynamics algorithm similar to that in \cite{Li2019LearningEM}. Our implicit sampler is a neural network comprising four inverse convolutional layers with hidden dimensions of 1024, 512, 256, and 3. Each layer incorporates 2D BatchNormalization and a LeakyReLU activation with a leak parameter of 0.2. The implicit generator's prior is a standard Multivariate Gaussian distribution. The score network adheres to a UNet architecture adapted from the repository\footnote{\url{https://github.com/huggingface/notebooks/blob/main/examples/annotated_diffusion.ipynb}}. To align with the multi-scale EBM design, the score function is parameterized as $S_\sigma(\xx) \coloneqq S(\xx)/\sigma$.

\paragraph{Training details.} 
The deep EBM is pre-trained on the MNIST dataset. The pre-trained EBM is then treated as a multi-scale un-normalized target distribution. We employ the DFT training method to refine the score network and generator, aiming to align the generator with the target distribution. Both the generator and score network are initialized randomly. During each training iteration, a noise level $\sigma$ is randomly selected within the range $[\sigma_{min}, \sigma_{max}]$. A sample is generated and Gaussian noise with variance $\sigma^2$ is added. The score network is then updated using denoising score matching on the generated samples. Subsequently, a batch of samples with the same variance of Gaussian noise is generated, and the generator's parameters are updated using the gradient estimation from Algorithm \ref{alg:denoise_training}. The Adam optimizer with a learning rate of 0.0001 is used for both networks. For the score network, the optimizer parameters are set to $\beta_0 = 0.9$ and $\beta_1 = 0.99$. For the generator, these are set to $\beta_0 = 0$ and $\beta_1 = 0.99$.

\paragraph{Evaluation metrics.} 
The Frechet Inception Distance (FID) \cite{heusel2017gans} is adapted to qualitatively assess the generated samples' quality. An MNIST image classifier with a WideResNet \cite{zagoruyko2016wide} architecture, having a depth of 16 and a widening factor of 8, is pre-trained. The Wasserstein distance in the feature space of this classifier is calculated using the FID method.

\end{document}